\documentclass[review]{elsarticle}
\usepackage[linesnumbered,procnumbered,ruled]{algorithm2e}
\usepackage{amsmath,amsthm}
\usepackage{amssymb}

\usepackage{graphicx} 
\usepackage{float} 
\usepackage{amsmath}
\usepackage{hyperref} 
\usepackage{amssymb}

\usepackage{multirow}

\usepackage{tcolorbox}

\DeclareMathOperator{\vect}{vec}

\newcommand{\abs}[1]{\left\lvert#1\right\rvert}

\makeatletter
\def\ps@pprintTitle{%
 \let\@oddhead\@empty
 \let\@evenhead\@empty
 \def\@oddfoot{}%
 \let\@evenfoot\@oddfoot}
\makeatother

\journal{Elseiver}

\begin{document}

\begin{frontmatter}

\title{Updating Singular Value Decomposition for Rank One Matrix Perturbation}

\author{Ratnik Gandhi}
\ead{ratnik.gandhi@ahduni.edu.in}
\author{Amoli Rajgor}
\ead{amoli.rajgor@iet.ahduni.edu.in}

\address{School of Engineering \& Applied Science, Ahmedabad University, Ahmedabad-380009, India}

\begin{abstract}
An efficient Singular Value Decomposition (SVD) algorithm is an important tool for distributed and streaming computation in big data problems. It is observed that update of singular vectors of a rank-1 perturbed matrix is similar to a Cauchy matrix-vector product. With this observation, in this paper, we present an efficient method for
updating Singular Value Decomposition of rank-1 perturbed matrix in $O(n^2 \ \text{log}(\frac{1}{\epsilon}))$ time. The method uses 
Fast Multipole Method (FMM) for updating singular vectors in $O(n \ \text{log} (\frac{1}{\epsilon}))$ time, where $\epsilon$ is the precision of computation. 
\end{abstract}


\end{frontmatter}


\section{Introduction}
	SVD is a matrix decomposition technique having wide range of applications, such as image/video compression, real time recommendation system, text mining (Latent Semantic Indexing (LSI)), signal processing, pattern recognition, etc. Computation  of  SVD  is rather simpler in case of centralized system where entire data matrix is available at a single location. It is more complex when the matrix is distributed over a network of devices. Further, increase in  complexity is attributed to the real-time arrival of new data. Processing of data over distributed stream-oriented systems require efficient algorithms that generate results and update them in real-time. In streaming environment the data is continuously updated and thus the output of any operation performed over the data must be updated accordingly. In this paper, a special case of SVD updating algorithm is presented where the updates to the existing data are of rank-1. \par
	
Organization of the paper is as follows: Section \ref{sec:Rank1SVD} motivates the problem of updating SVD for a rank-1 perturbation and presents a characterization to look at the problem from matrix-vector product point of view. An existing algorithm for computing fast matrix-vector product using interpolation is discussed in Section \ref{sec:FASTOriginal}. Section \ref{sec:IFASTAlgo} introduces 
Fast Multipole Method (FMM).
In Section \ref{sec:rank1SVDUpdate} we present an improved algorithm based on FMM for rank-1 SVD update that runs in $O(n^2 \log \frac{1}{\epsilon})$ time, where real $\epsilon >0$ is a desired accuracy parameter. Experimental results of the presented algorithms are given in Section \ref{sec:expResults}.
 
For completeness we have explained in details matrix factorization (\ref{app:MatFact}), solution to Sylvester Matrix (\ref{app:sylvestersolution}), FAST algorithm for Cauchy matrix vector product (\ref{app:fastAlgo}) and Fast Multipole Method (\ref{app:FMMAlgortihm}) at the end of this paper.\par
		
\section{Related Work}
In a series of work Gu and others \cite{gu93,ge93,ge94,ge95} present work on rank-one approximate SVD update. Apart from the low-rank SVD update, focus of their work is to discuss numerical computations and related accuracies in significant details. This leads to accurate computation of singular values, singular vectors and Cauchy matrix-vector product. Our work differs from this work as follows:  we use matrix factorization  that is explicitly based on solution of Sylvester equation \cite{stange08} to reach Eq. \eqref{eqn:UCTildeProduct} that computes updated singular vectors (see Section \ref{ssec:UpdateSingularVectors}). Subsequently, we reach an equation, Eq. \eqref{eqn:functionf(x)} (similar to equation (3.3) in \cite{ge93}). Further, we show that with this new matrix decomposition we reach the computational complexity $O(n^2 \log \frac{1}{\epsilon})$ for updating rank-1 SVD.

\section{SVD of Rank-One Perturbed Matrices}\label{sec:Rank1SVD}

	Let SVD of a $m \times n$ matrix $A = U\Sigma V^\top $.\footnote{We consider the model and characterization as in \cite{stange08}. A detailed factorization of this matrix is given in \ref{app:MatFact}.}  Where, $U \in \mathbb{R}^{m\times m}$, $\Sigma \in \mathbb{R}^{m\times n}$ and $V \in \mathbb{R}^{n\times n}$, where, without loss of generality, we assume, $m \leq n$. Let there be a rank-one update $ab^\top $ to matrix $A$ given by, 
	
	\begin{equation}\label{eqn:APurturbation1}
	\hat{A}  =  A + ab^\top  
	\end{equation}
	
	and let $\hat{U}\hat{\Sigma}\hat{V}^\top $ denote the new(updated) SVD, where $a \in \mathbb{R}^{m}$, $b \in \mathbb{R}^{n}$,
	
%
	Thus,
	
	\begin{equation}\label{eqn:AAT1}
	\hat{A}\hat{A}^\top = \hat{U}\hat{\Sigma}\hat{\Sigma}^\top \hat{U}^\top  .
	\end{equation}
	An algorithm for updating SVD of a rank-1 perturbed matrix 
			is given in Bunch and Nielsen \cite{bn78}. The algorithm updates singular values using characteristic polynomial and computes the updated singular vectors explicitly using the updated singular values.
	From (\ref{eqn:APurturbation1}),
	\begin{eqnarray}
	\hat{A}\hat{A}^\top  & = & (U \Sigma V^\top  +ab^\top )(U \Sigma V^\top  + ab^\top )^\top  \notag\\
	& = & U \Sigma V^\top V \Sigma^\top  U^\top  + \underbrace{U \Sigma V^\top b}_{\tilde{b}}a^\top  + a\underbrace{b^\top V \Sigma^\top U^\top }_{\tilde{b}^\top } + a\underbrace{b^\top b}_{\beta}a^\top \notag\\
	\hat{A}\hat{A}^\top  & = & U\Sigma\Sigma^\top U^\top  + \tilde{b}a^\top  + a\tilde{b}^\top  + \beta aa^\top\label{eqn:AAT3Rank11}
	\end{eqnarray}
	
		Where, $\tilde{b} = U \Sigma V^\top b $, \ $\tilde{b}^\top  = b^\top V \Sigma^\top U^\top $ and $\beta = b^\top b$.\\
	From (\ref{eqn:AAT3Rank11}) it is clear that the problem of rank-1 update (\ref{eqn:APurturbation1}) is modified to problem of three rank-1 updates (that is further converted to two rank-1 updates in (\ref{eqn:U+D+U+T1})).
	From (\ref{eqn:AAT1}) and (\ref{eqn:AAT3Rank11}) we get,

	\begin{eqnarray}
	\hat{U}\underbrace{\hat{\Sigma}\hat{\Sigma}^\top }_{\hat{D}}\hat{U}^\top  & = & U\underbrace{\Sigma\Sigma^\top }_{D}U^\top  + \tilde{b}a^\top  + a\tilde{b}^\top  + \beta aa^\top \notag\\
	\hat{U}\hat{D}\hat{U}^\top  & = & \underbrace{UDU^\top  + \rho_1a_1a_1^\top }_{\tilde{U}\tilde{D}\tilde{U}^\top } + \rho_2b_1b_1^\top . \label{eqn:U+D+U+T1}
	\end{eqnarray}
	Refer Appendix \ref{app:MatFact} Eq. (\ref{eqn:U+D+U+T}) for more details.
	Similar computation for right singular vectors is required.

			The following computation is to be done for each rank-1 update,  i.e., the procedure below will repeat four times, two times each for updating left and right singular vectors. From (\ref{eqn:U+D+U+T1}) we have,
			\begin{eqnarray}
			\tilde{U}\tilde{D}\tilde{U}^\top  & = & UDU^\top  + \rho_1a_1a_1^\top  \label{eqn:UDUTzRankoneUpdate1}\\
			& = & U\underbrace{(D + \rho_1\bar{a}\bar{a}^\top )}_{B}U^\top , \label{eqn:tildeUDU1}
			\end{eqnarray}
			where $\bar{a} = U^\top a_1$.
			From (\ref{eqn:tildeUDU1}) we have,
			\begin{eqnarray}
			B & := & D + \rho_1\bar{a}\bar{a}^\top \label{eqn:BDRank11}\\		
			B & = & \tilde{C}\tilde{D}\tilde{C}^\top  \text{~(Schur-decomposition).}\label{eqn:BTildeCDC1}	
			\end{eqnarray}
			From (\ref{eqn:BTildeCDC1}) and (\ref{eqn:tildeUDU1}) we get,
			
			\begin{eqnarray}
			\tilde{U}\tilde{D}\tilde{U}^\top   =  \underbrace{U(\tilde{C}}_{\tilde{U}}\tilde{D}\underbrace{\tilde{C}^\top )U^\top }_{\tilde{U}^\top }.
			\end{eqnarray}
			After adding the rank-1 perturbation to $UDU^\top$ in  (\ref{eqn:UDUTzRankoneUpdate1}), the updated singular vector matrix is given by matrix-matrix product
			\begin{equation}
			\tilde{U} = U \tilde{C} \label{eqn:UpdateUbyUCtilde1}.
			\end{equation} 
	
		Stange \cite{stange08}, extending
		the work of \cite{bn78}, presents an efficient way of updating SVD by exploring the underlying structure of the matrix-matrix computations of (\ref{eqn:UpdateUbyUCtilde1}). 

	\subsection{Updating Singular Values}\label{ssec:UpdateSingularvalues} 
	An approach for computing singular values is through eigenvalues. Given a matrix $\hat{S} = S + \rho uu^\top$, its eigenvalues $\tilde{d}$ can be computed in $O(n^2)$ numerical operations by solving characteristic polynomial of $(S + \rho uu^\top )x = \tilde{d}x$, where $S= \text{diag}(d_i)$. 
	Golub \cite{golub73} has shown that the above characteristic polynomial has following form. 
	
	\begin{equation}
			w(\tilde{d}) = 1 + \rho \sum_{i=1}^{n} \frac{u_i^2}{d_i - \tilde{d}}.\label{eqn:GolubEigenvalUpdate}
			\end{equation}

%
%
	Note that in the equation above $\tilde{d}$ is an unknown and thus, though the polynomial function is structurally similar to Eq. (\ref{eqn:functionf(x)}), we can not use FMM for solving it.
	
	Recall, in order to compute singular values $\hat{D}$ of updated matrix $\hat{A}$ we need to update $D$ twice as there are two symmetric rank-1 updates, i.e., for $$\tilde{U}\tilde{D}\tilde{U}^\top  = UDU^\top  + \rho_1a_1a_1^\top $$ 
	we will update $B = D + \rho_1\bar{a}\bar{a}^\top $ Eq. (\ref{eqn:tildeUDU1}) and similarly for 
	$$\hat{U}\hat{D}\hat{U}^\top  = \tilde{U}\tilde{D}\tilde{U}^\top  + \rho_2b_1b_1^\top $$ 
	we will update $B_1 = \tilde{D} + \rho_2\bar{b}\bar{b}^\top $. \\

	At times, while computing eigen-system, some of the eigenvalues and eigenvectors are known (This happens when there is some prior knowledge available about the eigenvalues or when the eigenvalues are approximated by using methods such as power iteration.). In such cases efficient SVD update method should focus on updating unknown eigen values and eigen vectors. \textit{Matrix Deflation} is the process of eliminating known eigenvalue from the matrix.
	
	Bunch, Nielsen and Sorensen \cite{bnc78} extended Golub's work by bringing the notion of deflation. They presented a computationally efficient method for computing singular values by deflating the system for cases: (1) when some values of $\bar{a}$ and $\bar{b}$ are zero, (2) $\abs{\bar{a}} = 1$ and $\abs{\bar{b}} = 1$ and (3) $B$ and $B_1$ has eigenvalues with repetition (multiplicity more than one). After deflation,
	the singular values of the updated matrix $\hat{D}$  can be obtained by Eq. (\ref{eqn:GolubEigenvalUpdate}).
	
%
	\subsection{Updating Singular Vectors}\label{ssec:UpdateSingularVectors}
	In order to update singular vectors the matrix-matrix product of (\ref{eqn:UpdateUbyUCtilde1}) is required. A naive method for matrix multiplication has complexity $O(n^3)$. We exploit matrix factorization in Stange \cite{stange08} that shows the structure of matrix $\tilde{C}$ to be Cauchy.\\
	
	From (\ref{eqn:BDRank11}) and (\ref{eqn:BTildeCDC1}) we get,	
\begin{eqnarray}
\tilde{C}\tilde{D}\tilde{C}^\top & = & D + \rho_1\bar{a}\bar{a}^\top  \notag\\
D\tilde{C} - \tilde{C}\tilde{D}& = & - \rho_1\bar{a}\bar{a}^\top \tilde{C}. \label{eqn:DC-CD}
\end{eqnarray}
Equation (\ref{eqn:DC-CD}) is Sylvester equation with solution,
\begin{equation}
(I_n \otimes D + (-\tilde{D})^\top  \otimes I_n)\vect \tilde{C} = \vect (- \rho_1\bar{a}\bar{a}^\top \tilde{C}).\label{eqn:SolnDC-CD}
\end{equation}

Simplifying L.H.S. of (\ref{eqn:SolnDC-CD}) for $\tilde{C}$ we get,
\begin{equation}
\tilde{C} = \left[
\begin{array}{ccc}
\bar{a}_1 && \\
&\ddots&\\
& & \bar{a}_n
\end{array} \right]
\left[
\begin{array}{ccc}
(\mu_1 - \lambda_1)^{-1} & \ldots & (\mu_n - \lambda_1)^{-1}\\
\vdots & &\vdots \\
(\mu_1 - \lambda_n)^{-1} & \ldots & (\mu_n - \lambda_n)^{-1}
\end{array} \right]
\left[
\begin{array}{ccc}
\bar{a}^\top c_1 && \\
&\ddots&\\
& & \bar{a}^\top c_n
\end{array} \right].
\end{equation}

Where, $\tilde{C} = \left[\begin{array}{ccc}c_1 &\ldots & c_n\end{array}\right]$ and from R.H.S. of (\ref{eqn:SolnDC-CD}), $c_i \in \mathbb{R}^n$ are  columns of the form, 
\begin{equation}
c_i = \rho_1 \left[\begin{array}{c} \frac{\bar{a}_1}{\mu_i - \lambda_1}\\\vdots\\\frac{\bar{a}_n}{\mu_i - \lambda_n}\end{array}\right] \bar{a}^\top c_i.\label{eqn:civector}
\end{equation}
Simplifying (\ref{eqn:civector}) further we get,

\begin{align*}
c_i & = \rho_1 \left[\begin{array}{c} \frac{\bar{a}_1}{\mu_i - \lambda_1}\\\vdots\\\frac{\bar{a}_n}{\mu_i - \lambda_n}\end{array}\right] \left[\begin{array}{cc}\bar{a}_1 & \bar{a}_2\end{array}\right] \left[\begin{array}{c}c_{i1}\\c_{i2}
\end{array}\right]\notag\\
c_i & = -\rho_1 \left(\bar{a}_1c_{i1}\ + \ \bar{a}_2c_{i2}\right) \left[\begin{array}{c} \frac{\bar{a}_1}{\lambda_1 - \mu_i}\\\vdots\\\frac{\bar{a}_n}{\lambda_n - \mu_i}\end{array}\right].\\
\end{align*}
Thus, denoting $\hat{\alpha}$ as $-\rho_1 \left(\bar{a}_1c_{i1}\ + \ \bar{a}_2c_{i2}\right)$ we get,

\begin{equation}
c_i = \hat{\alpha}\left[\begin{array}{ccc} \frac{\bar{a}_1}{\lambda_1 - \mu_i}\quad \ldots \quad \frac{\bar{a}_n}{\lambda_n - \mu_i}\end{array}\right]^\top .\label{eqn:ciSolnvector}
\end{equation}
 
Placing (\ref{eqn:ciSolnvector}) in (\ref{eqn:civector}) and we have,
\begin{align}
\hat{\alpha}\left[\begin{array}{c} \frac{\bar{a}_1}{\lambda_1 - \mu_i}\\ \vdots \\ \frac{\bar{a}_n}{\lambda_n - \mu_i}\end{array}\right] & = \hat{\alpha}\rho_1 \left[\begin{array}{c} \frac{\bar{a}_1}{\mu_i - \lambda_1}\\\vdots\\\frac{\bar{a}_n}{\mu_i - \lambda_n}\end{array}\right] \sum\limits_{k = 1}^{n} \frac{\bar{a}_k^2}{\lambda_k - \mu_i}.\notag
\end{align}
Therefore for each element of $c_i$,
\begin{align}
\hat{\alpha} \frac{\bar{a}_j}{\lambda_j - \mu_i} & = \hat{\alpha} \rho_1 \ \frac{\bar{a}_j}{\mu_i - \lambda_j}\ \sum\limits_{k = 1}^{n} \frac{\bar{a}_k^2}{\lambda_k - \mu_i}\notag\\
0 &= \hat{\alpha} \underbrace{\left(1 + \rho_1 \sum\limits_{k = 1}^{n} \frac{\bar{a}_k^2}{\lambda_k - \mu_i}\right)}_{=0}\label{eqn:indepndntAlpha}
\end{align}
Equation inside the bracket of (\ref{eqn:indepndntAlpha}) is the characteristic equation (\ref{eqn:GolubEigenvalUpdate}) used for finding updated eigenvalues $\mu_i$ of $\tilde{D}$ by using eigenvalues of $D$ i.e. $\lambda$. As $\mu_i$ are the zeros of this equation. Placing $\lambda_k$ and $\mu_i$ will make the term in the bracket (\ref{eqn:indepndntAlpha}) zero. Due to independence of the choice of scalar $\hat{\alpha}$, any value of $\hat{\alpha}$ can be used to scale the matrix $\bar{C}$. In order to make the final matrix orthogonal, each column of $\bar{C}$ is scaled by inverse of Euclidean norm of the respective column (\ref{eqn:CTildeFinalForm1}).

From Eq. (\ref{eqn:ciSolnvector}) matrix notation for $\tilde{C}$ can be written as

\begin{equation}
\tilde{C} = \underbrace{\left[
\begin{array}{ccc}
\bar{a}_1 && \\
&\ddots&\\
& & \bar{a}_n
\end{array} \right]
\overbrace{\left[\begin{array}{ccc} \frac{1}{\lambda_1 - \mu_1} \quad \cdots \quad \frac{1}{\lambda_1-\mu_n}\\ \vdots \hspace{0.5in}  \hspace{0.4in} \vdots\\ \frac{1}{\lambda_n - \mu_1}\quad \cdots \quad \frac{1}{\lambda_n - \mu_n} \end{array}\right]}^{C}}_{\bar{C} = \left[\begin{array}{ccc}c_1 &\ldots & c_n\end{array}\right]}
\left[
\begin{array}{ccc}
|c_1| && \\
&\ddots&\\
& & |c_n|
\end{array} \right]^{-1}. \label{eqn:CTildeFinalForm1}
\end{equation}
Where, $\abs{c_i}$ is the Euclidean norm of $c_i$.
From (\ref{eqn:CTildeFinalForm1}) it is evident that the $C$ matrix is similar to Cauchy matrix and $\tilde{C}$ is the scaled version of Cauchy matrix $C$. In order to update singular vectors we need to calculate matrix-matrix product as given in  (\ref{eqn:UpdateUbyUCtilde1}). From (\ref{eqn:CTildeFinalForm1}) and (\ref{eqn:UpdateUbyUCtilde1}) we get,
\begin{eqnarray}
U\tilde{C} & = & U\left[
\begin{array}{ccc}
\bar{a}_1 && \\
&\ddots&\\
& & \bar{a}_n
\end{array} \right]
\left[\begin{array}{ccc} \frac{1}{\lambda_1 - \mu_1} \quad \cdots \quad \frac{1}{\lambda_1-\mu_n}\\ \vdots \hspace{0.5in}  \hspace{0.4in} \vdots\\ \frac{1}{\lambda_n - \mu_1}\quad \cdots \quad \frac{1}{\lambda_n - \mu_n} \end{array}\right] 
\left[
\begin{array}{ccc}
|c_1| && \\
&\ddots&\\
& & |c_n|
\end{array} \right]^{-1}\notag
\end{eqnarray}

\begin{eqnarray}
U\tilde{C} & = & \left[\begin{array}{ccc}\hat{u}_1 &\ldots & \hat{u}_n\end{array}\right] C \left[
\begin{array}{ccc}
|c_1| && \\
&\ddots&\\
& & |c_n|
\end{array} \right]^{-1}.\label{eqn:U1CTrummers}\\
\text{\qquad Where,~} \hat{u}_i & = & u_i\bar{a}_i \ and \ i = 1, \ldots ,n.\notag\\
U\tilde{C} & = & U_1  C \left[
\begin{array}{ccc}
|c_1| && \\
&\ddots&\\
& & |c_n|
\end{array} \right]^{-1}\notag\\
U\tilde{C} & = & U_2 \left[
\begin{array}{ccc}
|c_1| && \\
&\ddots&\\
& & |c_n|
\end{array} \right]^{-1}\label{eqn:UCTildeProduct}\\
\text{\qquad Where,~} U_1&  = & \left[\begin{array}{ccc}\hat{u}_1 &\ldots & \hat{u}_n\end{array}\right]\text{\ and~} U_2 = U_1 C \notag.
\end{eqnarray} 

	\subsubsection{Trummer's Problem: Cauchy Matrix-Vector Product}\label{sec:trummers}
	
	In (\ref{eqn:U1CTrummers}) there are $n$ vectors in $U$ which are multiplied with Cauchy matrix $C$. The problem of multiplying a Cauchy matrix with a vector is called Trummer's problem. As there are $n$ vectors in (\ref{eqn:U1CTrummers}), matrix-matrix product in it can be represented as $n$ matrix-vector products, i.e., it is same as solving Trummer's problem $n$ times. Section \ref{sec:FASTOriginal} describes an algorithm given in \cite{gerasoulis88} which efficiently computes such matrix-vector product in $O(n \ \text{log}^2(n))$ time.

	\section{FAST: Method based on Polynomial Interpolation and FFT}\label{sec:FASTOriginal} 
	Consider the the matrix - (Cauchy) matrix product $U_2 = U_1C$ of Section \ref{ssec:UpdateSingularVectors}. This product can be written as

	\begin{eqnarray}
	U_2 & = & \left[\begin{array}{ccc}\hat{u}_1 &\ldots & \hat{u}_n\end{array}\right] C\\
	U_2 & = & \left[\begin{array}{ccc} u_{11} \quad \cdots \quad u_{1n}\\ \vdots \hspace{0.3in} \vdots \hspace{0.3in} \vdots \\ u_{n1} \quad \cdots \quad u_{nn}\end{array}\right]\left[\begin{array}{ccc} \frac{1}{\lambda_1 - \mu_1} \quad \cdots \quad \frac{1}{\lambda_1-\mu_n}\\ \vdots \hspace{0.4in} \vdots \hspace{0.4in} \vdots\\ \frac{1}{\lambda_n - \mu_1}\quad \cdots \quad \frac{1}{\lambda_n - \mu_n} \end{array}\right].
	\end{eqnarray}
	The dot product of each row vector of $U_1$ and the coulmn vector of the Cauchy matrix can be represented in terms of a function of eigenvalues $\mu_i$,  $i = 1,\ldots,n$,
	
	\begin{eqnarray} 
	f(\mu_1) & = & \left[\begin{array}{ccc} u_{11} \quad \cdots \quad u_{1n}\end{array}\right]\left[\begin{array}{c}\frac{1}{\lambda_1 - \mu_1}\\\vdots\\\frac{1}{\lambda_n - \mu_1}\end{array}\right]\\
	& = & \frac{u_{11}}{\lambda_1 -\mu_1}+\ldots+\frac{u_{1n}}{\lambda_n -\mu_1}\notag\\ 
	f(\mu_1)& = & \sum_{j=1}^{n} \frac{u_{1j}}{\lambda_j - \mu_1}\notag.\\
	\text{Hence in general,}&&\notag\\
	f(x) & = & \sum_{j=1}^{n} \frac{u_j}{\lambda_j - x}.	\label{eqn:functionf(x)}
	\end{eqnarray}
	Equation (\ref{eqn:functionf(x)}) can be shown as the ratio of two polynomials 
	\begin{eqnarray} 
		f(x) & = & \frac{h(x)}{g(x)},\label{eqn:ratioh/g}
		\text{\qquad where, }  g(x) = \prod\limits_{j=1}^{n} (\lambda_j - x).\label{eqn:Defineg(x)}\\
		h(x) &=& g(x)\sum_{j=1}^{n} \frac{u_j}{\lambda_j - x}\\
		h(x) &=& \prod\limits_{j=1}^{n} (\lambda_j - x)\sum_{j=1}^{n} \frac{u_j}{\lambda_j - x}
		\end{eqnarray}

		The FAST algorithm  \cite{gerasoulis88} represents the a matrix-vector product as (\ref{eqn:functionf(x)}) and (\ref{eqn:ratioh/g}). It then finds solutions to this problem by the use of interpolation.		
	The FAST algorithm of Gerasoulis \cite{gerasoulis88} has been reproduce in \ref{app:fastAlgo}.

	\section{Matrix-vector product using FMM} \label{sec:IFASTAlgo}
	The FAST algorithm  in \cite{gerasoulis88} has complexity $O (n \ \text{log}^2n)$. It computes the function $f(x)$ using FFT and interpolation. We observe that these two methods are two major procedures that contributes to the overall complexity of FAST algorithm. To reduce this complexity we present an algorithm that uses FMM for finding Cauchy matrix-vector product that updates the SVD of rank-1 perturbed matrix with time complexity  $O(n^2 \ \text{log}(\frac{1}{\epsilon}))$.
	
	Recall from Section \ref{ssec:UpdateSingularVectors}, update of singular vectors require matrix-(Cauchy) matrix product $U_2 = U_1C$, This product is represented as the function below.
	$$ f(\mu_i)  =  \sum_{j=1}^{n} \frac{-u_j}{\mu_i - \lambda_j}.$$
	We use FMM to compute this function. 
	

	\subsection{Fast Multipole Method (FMM)}\label{ssec:FMM}
	An algorithm for computing potential and force of $n$ particles in a system is given in Greengard and Rokhlin \cite{gr87}. This algorithm enables fast computation of forces in an $n$-body problem by computing interactions among points in terms of far-field and local expansions. It starts with clustering particles in groups such that the inter-cluster and intra-cluster distance between points make them well-separated. Forces between the clusters are computed using Multipole expansion.

	Dutt et al. \cite{dgr96} describes the idea of FMM for particles in one dimension and presents an algorithm for evaluating sums of the form, 
						
	\begin{equation}
	f(x) = \sum_{k=1}^{N}\alpha_k \cdot \phi(x - x_k).\label{eqn:f(x)Define}
	\end{equation}
						
	Where, $\{\alpha_1,\ldots,\alpha_N\}$ is a set of complex numbers and $\phi(x)$ can be a function that is singular at $x=0$ and smooth every where. Based on the choice of function, (\ref{eqn:f(x)Define}) can be used to evaluate sum of different forms for example:
										
	\begin{equation}\label{eqn:f(x)_SVDU}
	f(x) = \sum_{k=1}^{N}\frac{\alpha_k}{x-x_k}
	\end{equation}
	where, $\phi(x) = \frac{1}{x}$.
	Dutt et al. \cite{dgr96} presents an algorithm for computing  summation (\ref{eqn:f(x)_SVDU}) using FMM that runs in $O(n \ \text{log}(\frac{1}{\epsilon}))$. 
			

\subsection{Summation using FMM}\label{ssec:SumFMM}
			
Consider \textit{well separated} points $\{x_1,x_2,\ldots, x_N\}$ and $\{y_1,\ldots,y_M\}$ in $\mathbb{R}$ such that for points $x_0,y_0 \in \mathbb{R}$ and $r>0,\ r \in \mathbb{R}$.
\begin{eqnarray*}
|x_i - x_0| & < & r \ \forall i = 1,\ldots,N\\
|y_i - y_0| & < & r \ \forall i = 1,\ldots,M \text{\ and}\\
|x_0 - y_0| & > & 4r\\
\end{eqnarray*}
For a function defined as $f:\mathbb{R} \rightarrow \mathbb{C}$ such that
				\begin{equation}
					f(x) = \sum_{k=1}^{N}\frac{\alpha_k}{x-x_k}\label{eqn:f(x)asSum}
					\end{equation}
where, $\{\alpha_1 , \ldots, \alpha_k\}$ is a set of complex numbers. Given $f(x)$, the task is to find $f(y_1),\ldots,f(y_M)$.

\section{Rank-One SVD update}\label{sec:rank1SVDUpdate}	
In this section, we present Algorithm \ref{algo:Rank-1 SVD update} that uses FMM and updates Singular Value Decomposition in $O(n^2 \ \text{log}(\frac{1}{\epsilon}))$ time.

\subsection{\textbf{Algorithm:} Update SVD of rank-1 modified matrix using FMM}\label{algo:Rank-1 SVD update}
\begin{enumerate}
\item[INPUT] $A \in \mathbb{R}^{m \times n} = U\Sigma V^T,\ a\in \mathbb{R}^{m} \text{and}\ b\in \mathbb{R}^{n}$

\item[OUTPUT] $U_n,\ \Sigma_n,\ \text{and}\ V_n$

\item[STEP 1] Compute $\tilde{b} = U\Sigma V^Tb,\ \tilde{a} = V\Sigma^TU^Ta,\ \beta = b^Tb,\ \alpha = a^Ta,\ D_u = \Sigma\Sigma^T$ and $D_v = \Sigma^T\Sigma$.

\item[STEP 2] Compute the Schur decomposition of $\left[\begin{array}{cc}\beta \quad 1 \\ 1 \quad 0\end{array}\right] = Q_u \left[\begin{array}{cc}\rho_1 \quad 0 \\ 0 \quad \rho_2\end{array}\right]Q_u^T$ and $\left[\begin{array}{cc}a \ \tilde{b}\end{array}\right]Q_u = \left[\begin{array}{cc}a_1 \ b_1 \end{array}\right]$.
\item[STEP 3] Compute the Schur decomposition of $\left[\begin{array}{cc}\alpha \quad 1 \\ 1 \quad 0\end{array}\right] = Q_v \left[\begin{array}{cc}\rho_1 \quad 0 \\ 0 \quad \rho_2\end{array}\right]Q_v^T$ and $\left[\begin{array}{cc}b \ \tilde{a}\end{array}\right]Q_v = \left[\begin{array}{cc}a_2 \ b_2 \end{array}\right]$.

\item[STEP 4] Compute updated left singular vector $\tilde{U}\ \text{and} \ \tilde{D}$ by calling the procedure RankOneUpdate$(U,a_1,D,\rho_1)$ - Algorithm \ref{algo:RankOneUpdate}.

\item[STEP 5] Compute left singular vector of $\hat{A}$, i.e., $U_n \ \text{and} \ D_n$ by calling the procedure RankOneUpdate$(\tilde{U},b_1\tilde{D},\rho_2)$.

\item[STEP 6] Compute updated right singular vector $\tilde{V}  \ \text{and} \ \tilde{D_v}$ by calling the procedure RankOneUpdate$(V,a_2,D_v,\rho_3)$.

\item[STEP 7] Compute right singular vector of $\hat{A}$ i.e. $V_n  \ \text{and} \ D_{vn}$ by calling the procedure RankOneUpdate$(\tilde{V},b_2,\tilde{D_v},\rho_4)$.

\item[STEP 8] Find singular values by computing square root of the updated eigenvalues $\Sigma_n = \sqrt{D_n}$.
\end{enumerate}
Note that the Schur decomposition of Steps 2 and 3 are computed over constant size matrices and thus the decomposition will take constant time.
\subsection{\textbf{Algorithm:} RankOneUpdate}
\label{algo:RankOneUpdate}
\begin{enumerate}
\item[INPUT] $U,a_1,D,\rho_1$
\item[OUTPUT] $\tilde{U},\ \tilde{D}$

\item[STEP 1] Compute $\bar{a} = U^T a_1$. 



\item[STEP 2] Compute $\mu$ as zeros of the equation $\ w(\mu) = 1 + \rho_1 \sum\limits_{i=1}^{n} \frac{\bar{a}_i^2}{\lambda_i - \mu}$.

\item[STEP 3] Compute $C = \left[\begin{array}{ccc} \frac{1}{\lambda_1 - \mu_1} \quad \cdots \quad \frac{1}{\lambda_1-\mu_n}\\ \vdots \hspace{0.5in}  \hspace{0.4in} \vdots\\ \frac{1}{\lambda_n - \mu_1}\quad \cdots \quad \frac{1}{\lambda_n - \mu_n} \end{array}\right]$.

\item[STEP 4] Compute $U_1 = U \scriptstyle\left[\begin{array}{ccc}\bar{a}_1 && \\&\ddots&\\ & & \bar{a}_n\end{array} \right]$

\item[STEP 5]Compute $\bar{C} = \scriptstyle\left[\begin{array}{ccc}\bar{a}_1 && \\&\ddots&\\ & & \bar{a}_n\end{array} \right]~\scriptstyle\left[\begin{array}{ccc} \frac{1}{\lambda_1 - \mu_1} \quad \cdots \quad \frac{1}{\lambda_1-\mu_n}\\ \vdots \hspace{0.5in}  \hspace{0.4in} \vdots\\ \frac{1}{\lambda_n - \mu_1}\quad \cdots \quad \frac{1}{\lambda_n - \mu_n} \end{array}\right]$

\item[STEP 6] Compute $U_2= U_1C$ as $n$ matrix-vector product. Where each row-column dot product is represented as a function 
$$ f(\mu_i)  =  \sum_{j=1}^{n} \frac{-u_j}{\mu_i - \lambda_j }.$$ for each $\mu_i$.

$U_2 = \textsc{FMM}(\lambda,\ \mu,\ -U_1)$ - (\ref{app:FMMAlgortihm}).

\item[STEP 7] Form $\tilde{U}$ by dividing each column of $U_2$ by norm of respective column of $\bar{C}$.
\end{enumerate}

\subsection{Complexity for Rank-One SVD update}
	\newtheorem{theorem}{Theorem}
			\begin{theorem}
			Given a matrix $A \in \mathbb{R}^{m \times n}$ such that $m \leq n$ and precision of computation parameter $\epsilon>0$, the complexity of computing SVD of a rank-1 update of $A$ with Algorithm \ref{algo:Rank-1 SVD update} is $O(n^2 \ \text{log} (\frac{1}{\epsilon}))$.
			\end{theorem}
			
			\begin{proof}
			Follows from Algorithm \ref{algo:Rank-1 SVD update} and Table \ref{table:SVDusingModifiedFast}.
			\end{proof}
				\begin{table}[h]
			\centering 
				\begin{tabular}{|c|l|l|}
					\hline\rule[-2mm]{0mm}{8mm} Section & Complexity & Operation \\ 
					\hline\rule{0pt}{4ex}\ref{sec:Rank1SVD} & $O(n^2)$ & Compute $\hat{U}\hat{\Sigma}\hat{\Sigma}^\top \hat{U}^\top $ and $\hat{V}\hat{\Sigma}^\top \hat{\Sigma}\hat{V}^\top $\\ 
\hline\rule{0pt}{4ex}\ref{ssec:UpdateSingularvalues} & $O(n^2)$ & Update singular values solving \\
					&& 	$w(t) = 1 + \rho \sum\limits_{j=1}^{n} \frac{z_j^2}{\lambda_j - t} = 0$\\ 
\hline\rule{0pt}{4ex}\ref{ssec:FMM} & $O(n \ \text{log}(\frac{1}{\epsilon}))$ & Update singular vectors as $\tilde{U} = U \tilde{C}$ \\
					&& using FMM\\
					\hline\rule{0pt}{4ex} &$O(n^2 \ \text{log}(\frac{1}{\epsilon}))$& Total Complexity\\
					\hline
				\end{tabular}
				\caption{Rank-1 SVD update Complexity}
				\label{table:SVDusingModifiedFast}				
			\end{table}

\section{Numerical Results}\label{sec:expResults}
All the computations for the algorithm were performed on MATLAB over a machine with Intel i5, quad-core, 1.7 GHz, 64-bit processor with 8 GB RAM. Matrices used in these experiments are square and generated randomly with values ranging from $[1,9]$. The sample size varies from $2 \times 2$ to $35 \times 35$. 
For computing the matrix vector product we use FMM Algorithm (instead of FAST Algorithm) with machine precision $\epsilon$  =$5^{-10}$. 
%

In order to update singular vectors we need to perform two rank-1 updates (\ref{eqn:U+D+U+T}). For each such rank-1 update we use FMM. 

\begin{align}
\hat{U}\hat{D}\hat{U}^T & =  \underbrace{UDU^T + \rho_1a_1a_1^T}_{\tilde{U}\tilde{D}\tilde{U}^T} + \rho_2b_1b_1^T.\label{eqn:U+D+U+T}
\end{align}
Figure \ref{fig:TimeFmmStngHalf_35} shows the time taken by FAST Algorithm and FMM Algorithm to compute the first rank-1 update. 

\begin{figure}[H]    
\begin{minipage}[t]{0.45\textwidth}
\includegraphics[width=\linewidth]{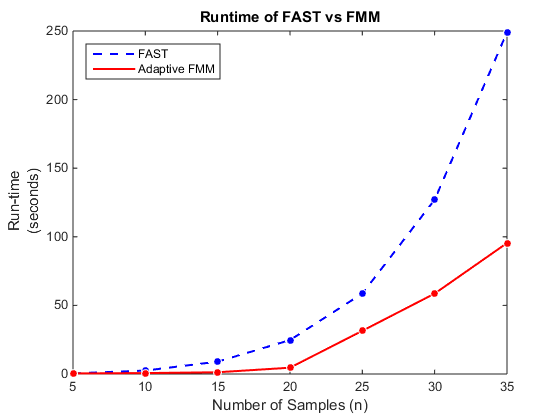}
\caption{Run-time for rank-1 update (\ref{eqn:U+D+U+T})}
\label{fig:TimeFmmStngHalf_35}
\end{minipage}
\hspace{\fill}
\begin{minipage}[t]{0.45\textwidth}
\includegraphics[width=\linewidth]{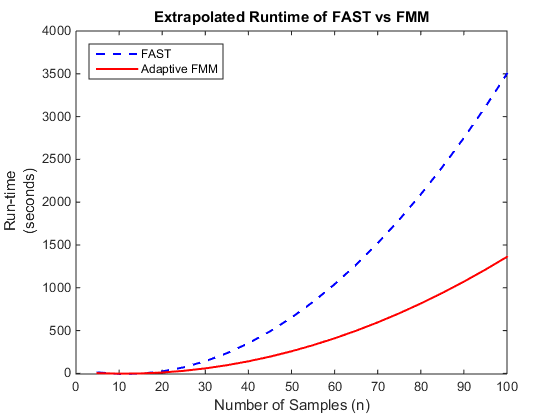}
\caption{Extrapolated run-time for rank-1 update (\ref{eqn:U+D+U+T})}
\label{fig:ExtrTimeFmmStngHalf_100}
\end{minipage}
\end{figure}

\subsection{Choice of $\epsilon$}
In earlier computations we fixed the machine precision parameter $\epsilon$ for FMM based on the order of the Chebyshev polynomials, i.e., $\epsilon = 5^{-p}$, where $p$ is the order of Chebyshev polynomial. As we increase the order of polynomials - we expect reduction in the error and increase in updated vectors accuracy. This increase in accuracy comes with the cost of higher computational time. To show this we fix the input matrix dimension to $25 \times 25$ and generate values of these matrices randomly in the range [0,1]. Figure \ref{fig:ErrVsP_25} shows the error between updated singular vectors generated by Algorithm \ref{algo:Rank-1 SVD update} and exact computation of singular vector updates. 
The results in figure justifies our choice of fixed machine precision $p = 20$.

\begin{figure}[H]
	\begin{center} 
	\includegraphics[scale = 0.70]{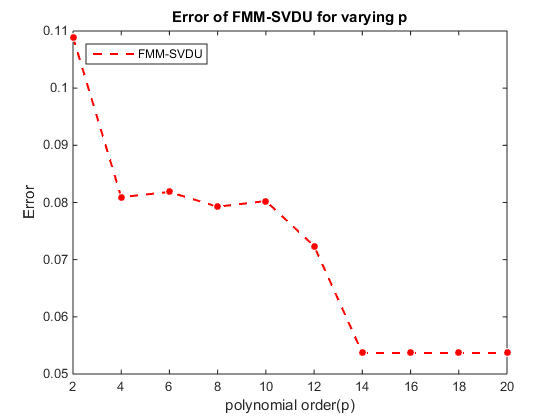}
	\caption{Error with varying $p$}
	\label{fig:ErrVsP_25}
	\end{center}
\end{figure}

Error is computed using the equation (\ref{eqn:error}) \cite{stange08}, where $\hat{A}$ is the perturbed matrix, $\hat{U}\hat{\Sigma}\hat{V}^T$ is the approximation computed using FMM-SVDU and max $\hat{\sigma}$ is the directly computed maximum eigenvalue of $\hat{A}$. 

\begin{align}
\text{Error} & = \text{max}\left\lvert\frac{\hat{A} - \hat{U}\hat{\Sigma}\hat{V}^T}{\text{max}\ \hat{\sigma}}\right\rvert \label{eqn:error}
\end{align}

Table \ref{table:Fmm-svduAccurcy} shows the accuracy of rank-1 SVD update using FMM- SVDU for varying sample size. Figure \ref{fig:Fmm-svduErr} shows plot of the accuracy of FMM-SVDU with varying sample size. 

\begin{table}[H]
	\centering
		\caption{rank-1 SVD update accuracy}
		\label{table:Fmm-svduAccurcy}
		\begin{tabular}{|c|c|c|}
		\hline\rule[-2mm]{0mm}{8mm} No. & Dimension & Error\\ 
		\rule[-2mm]{0mm}{8mm} &  $n \times n$ & (Eqn.  \ref{eqn:error}) \\ 
		\hline\rule{0pt}{5ex}1. & $10 \times 10$ & 0.141245710607176\\ 
		\hline\rule{0pt}{5ex}2. & $20 \times 20$ & 0.0837837759946002\\
		\hline\rule{0pt}{5ex}3. & $30 \times 30$ & 0.0559656608985486\\
		\hline\rule{0pt}{5ex}4. & $40 \times 40$ & 0.0623799282154490\\
		\hline\rule{0pt}{5ex}5. & $50 \times 50$ & 0.0464500903310721\\
		\hline 
		\end{tabular}
	\end{table}
		
\begin{figure}[H]
	\begin{center} 
	\includegraphics[scale =0.70]{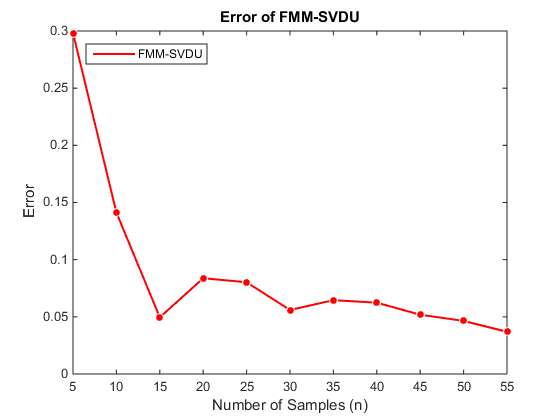}
	\caption{Accuracy of FMM-SVDU}
	\label{fig:Fmm-svduErr}
	\end{center}
\end{figure}


		\section{Conclusion}
		In this paper we considered the problem of updating Singular Value Decomposition of a rank-1 perturbed matrix. We presented an efficient algorithm for updating SVD of rank-1 perturbed matrix that uses the Fast Multipole method for improving Cauchy matrix- vector product computational efficiency. 
An interesting and natural extension of this work is to consider updates of rank-k.\\
				
%
\section*{References}



\appendix
\section{Matrix Factorization}\label{app:MatFact}
We consider the matrix factorization in \cite{stange08}. Let SVD of a $m \times n$ matrix $A = U\Sigma V^\top $. Where, $U \in \mathbb{R}^{m\times m}$, $\Sigma \in \mathbb{R}^{m\times n}$ and $V \in \mathbb{R}^{n\times n}$. Here it is assumed that $m \leq n$. Let there be a rank-one update $ab^\top $ to matrix $A$ Eq. (\ref{eqn:APurturbation}) and let $\hat{U}\hat{\Sigma}\hat{V}^\top $ denote the new(updated) SVD, where $a \in \mathbb{R}^{m}$, $b \in \mathbb{R}^{n}$.
	
	\begin{eqnarray}
	\hat{A} & = & A + ab^\top  \label{eqn:APurturbation}\\
	\hat{A}& = & \hat{U} \hat{\Sigma} \hat{V}^\top 
	\end{eqnarray}
	
	In order to find SVD of this updated matrix $\hat{A}$ we need to compute $\hat{A}\hat{A}^\top $ and $\hat{A}^\top \hat{A}$ because left singular vector $\hat{U}$ of $\hat{A}$ is orthonormal eigenvector of $\hat{A}\hat{A}^\top $ and right singular vector $\hat{V}$ of $\hat{A}$ is orthonormal eigenvector of $\hat{A}^\top \hat{A}$.
	
	\begin{eqnarray}
	\hat{A}\hat{A}^\top  & = & (\hat{U}\hat{\Sigma}\hat{V}^\top )(\hat{U}\hat{\Sigma}\hat{V}^\top )^\top  \\
	& = & \hat{U}\hat{\Sigma}\hat{V}^\top  \hat{V}\hat{\Sigma}^\top \hat{U}^\top \notag\\
	\hat{A}\hat{A}^\top & = & \hat{U}\hat{\Sigma}\hat{\Sigma}^\top \hat{U}^\top  \text{\qquad Where,~} \hat{V}^\top \hat{V} = I\label{eqn:AAT}
	\end{eqnarray}
	
	\begin{eqnarray}
	\hat{A}\hat{A}^\top  & = & (U \Sigma V^\top  +ab^\top )(U \Sigma V^\top  + ab^\top )^\top  \notag\\
	& = & (U \Sigma V^\top  +ab^\top )(V \Sigma^\top  U^\top  + ba^\top )\notag\\
	& = & U \Sigma V^\top V \Sigma^\top  U^\top  + \underbrace{U \Sigma V^\top b}_{\tilde{b}}a^\top  + a\underbrace{b^\top V \Sigma^\top U^\top }_{\tilde{b}^\top } + a\underbrace{b^\top b}_{\beta}a^\top \notag\\
	\hat{A}\hat{A}^\top  & = & U\Sigma\Sigma^\top U^\top  + \tilde{b}a^\top  + a\tilde{b}^\top  + \beta aa^\top \text{\qquad Where,~} V^\top V = I\label{eqn:AAT3Rank1}
	\end{eqnarray}
	
		Where, $\tilde{b} = U \Sigma V^\top b $, \ $\tilde{b}^\top  = b^\top V \Sigma^\top U^\top $ and $\beta = b^\top b$.\\
	From Eq. (\ref{eqn:AAT3Rank1}) it is clear that the problem of rank-1 update Eq. (\ref{eqn:APurturbation}) is modified to problem of three rank-1 updates that is further converted to two rank-1 updates in Eq. (\ref{eqn:U+D+U+T}).
	Equating Eq. (\ref{eqn:AAT}) and Eq. (\ref{eqn:AAT3Rank1}) we get,

	\begin{eqnarray}
	\hat{U}\underbrace{\hat{\Sigma}\hat{\Sigma}^\top }_{\hat{D}}\hat{U}^\top  & = & U\underbrace{\Sigma\Sigma^\top }_{D}U^\top  + \tilde{b}a^\top  + a\tilde{b}^\top  + \beta aa^\top \notag\\
	\hat{U}\hat{D}\hat{U}^\top & = & UDU^\top + \tilde{b}a^\top  + a\tilde{b}^\top  + \beta aa^\top \notag\\
	& = & UDU^\top + \beta aa^\top  + \tilde{b}a^\top  + a\tilde{b}^\top \notag\\
	& = & UDU^\top + (\beta a + \tilde{b})a^\top  + a\tilde{b}^\top \notag\\
	& = & UDU^\top + \left[\begin{array}{cc}\beta a + \tilde{b} \quad a \end{array}\right]  \left[\begin{array}{c}a^\top \\ \tilde{b}^\top \end{array}\right]\notag
	\end{eqnarray}
	\begin{eqnarray}
		\hat{U}\hat{D}\hat{U}^\top  & = & UDU^\top  + \left[\begin{array}{cc} a \quad \tilde{b}\end{array}\right] \left[\begin{array}{cc}\beta \quad 1 \\ 1 \quad 0\end{array}\right] \left[\begin{array}{c} a^\top  \\ \tilde{b}^\top \end{array}\right]\notag
		\end{eqnarray}
	\begin{eqnarray}
		\hat{U}\hat{D}\hat{U}^\top  & = & UDU^\top  + \left[\begin{array}{cc} a \quad \tilde{b}\end{array}\right] Q \left[\begin{array}{cc}\rho_1 \quad 0 \\ 0 \quad \rho_2\end{array}\right]Q^\top  \left[\begin{array}{c} a^\top  \\ \tilde{b}^\top \end{array}\right]\notag
		\end{eqnarray}
	\begin{eqnarray}
	\text{Where,~} && \left[\begin{array}{cc}\beta \quad 1 \\ 1 \quad 0\end{array}\right]= Q \left[\begin{array}{cc}\rho_1 \quad 0 \\ 0 \quad \rho_2\end{array}\right]Q^\top \text{~(Schur-decomposition)}\notag
	\end{eqnarray}
	\begin{eqnarray}
	\hat{U}\hat{D}\hat{U}^\top  & = & UDU^\top  + \underbrace{\left[\begin{array}{cc} a \quad \tilde{b}\end{array}\right] Q}_{\left[\begin{array}{cc} a_1 \quad b_1\end{array}\right]} \left[\begin{array}{cc}\rho_1 \quad 0 \\ 0 \quad \rho_2\end{array}\right]\underbrace{Q^\top  \left[\begin{array}{c} a^\top  \\ \tilde{b}^\top \end{array}\right]}_{\left[\begin{array}{cc} a_1 \quad b_1\end{array}\right]^\top }\notag\\
	\hat{U}\hat{D}\hat{U}^\top  & = & UDU^\top  + \left[\begin{array}{cc} a_1 \quad b_1\end{array}\right] \left[\begin{array}{cc}\rho_1 \quad 0 \\ 0 \quad \rho_2\end{array}\right]\left[\begin{array}{c} a_1^\top  \\ b_1^\top \end{array}\right]^\top \notag\\
	\hat{U}\hat{D}\hat{U}^\top  & = & \underbrace{UDU^\top  + \rho_1a_1a_1^\top }_{\tilde{U}\tilde{D}\tilde{U}^\top } + \rho_2b_1b_1^\top . \label{eqn:U+D+U+T}
	\end{eqnarray}
		Similarly for computing right singular vectors do the following.
		\begin{eqnarray}
			\hat{A}^\top \hat{A} & = & (U \Sigma V^\top  +ab^\top )^\top (U \Sigma V^\top  + ab^\top ) \notag\\
			& = & (V \Sigma^\top  U^\top  + ba^\top )(U \Sigma V^\top  +ab^\top )\notag\\
			\hat{A}^\top \hat{A}& = & V \Sigma^\top  U^\top U \Sigma V^\top  + \underbrace{V \Sigma^\top  U^\top a}_{\tilde{a}}b^\top  + b\underbrace{a^\top U \Sigma V^\top }_{\tilde{a}^\top }+ b\underbrace{a^\top a}_{\alpha}b^\top \notag\\
			(\hat{V}\hat{\Sigma}^\top \hat{U}^\top )(\hat{U}\hat{\Sigma}\hat{V}^\top )& = & V\Sigma^\top \Sigma V^\top  + \tilde{a}b^\top  + b\tilde{a}^\top  + \alpha bb^\top \text{\qquad Where,~} U^\top U = I \notag.\\
			\hat{V}\hat{\Sigma}^\top \hat{\Sigma}\hat{V}^\top  & = &  V\Sigma^\top \Sigma V^\top  + \rho_3a_2a_2^\top  + \rho_4b_2b_2^\top  \label{eqn:V+D+V+T}
			\end{eqnarray}
			
			The following computation is to be done for each rank-1 update in Eq. (\ref{eqn:U+D+U+T}) and Eq. (\ref{eqn:V+D+V+T}) i.e. the below procedure will repeat four times, two times each for updating left and right singular vectors. From Eq. (\ref{eqn:U+D+U+T}) we have,
			\begin{eqnarray}
			\tilde{U}\tilde{D}\tilde{U}^\top  & = & UDU^\top  + \rho_1a_1a_1^\top  \label{eqn:UDUTzRankoneUpdate}\\
			& = & UDU^\top  + \rho_1(UU^\top a_1)(UU^\top a_1)^\top \notag\\
			& = & UDU^\top  + \rho_1(U\bar{a})(U\bar{a})^\top  \text{~Where,~} \bar{a} = U^\top a_1\notag\\
			& = & UDU^\top  + \rho_1(U\bar{a} \bar{a}^\top U^\top )\notag\\
			\tilde{U}\tilde{D}\tilde{U}^\top  & = & U\underbrace{(D + \rho_1\bar{a}\bar{a}^\top )}_{B}U^\top . \label{eqn:tildeUDU}
			\end{eqnarray}
			From Eq. (\ref{eqn:tildeUDU}) we have,
			\begin{eqnarray}
			B & := & D + \rho_1\bar{a}\bar{a}^\top \label{eqn:BDRank1}\\		
			B & = & \tilde{C}\tilde{D}\tilde{C}^\top  \text{~(Schur-decomposition).}\label{eqn:BTildeCDC}	
			\end{eqnarray}
			Placing Eq. (\ref{eqn:BTildeCDC}) in Eq. (\ref{eqn:tildeUDU}) we get,
			
			\begin{eqnarray}
			\tilde{U}\tilde{D}\tilde{U}^\top   =  \underbrace{U(\tilde{C}}_{\tilde{U}}\tilde{D}\underbrace{\tilde{C}^\top )U^\top }_{\tilde{U}^\top }.
			\end{eqnarray}
			After the rank-1 update to $UDU^\top $ Eq. (\ref{eqn:UDUTzRankoneUpdate}) the updated singular vector matrix is given by matrix-matrix product
			\begin{equation}
			\tilde{U} = U \tilde{C} \label{eqn:UpdateUbyUCtilde}.
			\end{equation}

\section{Solution to Sylvester Equation}\label{app:sylvestersolution}
In Section \ref{ssec:UpdateSingularVectors} we discussed method for updating singular vectors. For the same we derived solutions to (\ref{eqn:DC-CD}) using Sylvester equation. In this section we present details of how this solution can be obtained.\\

For simplicity consider the case where the dimension of all the matrices ($C$, $D$ and $\tilde{D}$) is $2 \times 2$.

\begin{align}
L.H.S & = (I_n \otimes D + (-\tilde{D})^\top  \otimes I_n)\vect \tilde{C}\notag\\
& = \left\{\left(\left[\begin{array}{cc}1 & 0\\ 0 & 1\end{array}\right] \otimes \left[\begin{array}{cc}\lambda_1 & 0\\ 0 & \lambda_2\end{array}\right]\right)  - \left(\left[\begin{array}{cc}\mu_1 & 0\\ 0 & \mu_2\end{array}\right]^\top \otimes \left[\begin{array}{cc}1 & 0\\ 0 & 1\end{array}\right]\right)\right\}\vect \left[\begin{array}{cc}c_{11} & c_{12}\\ c_{21} & c_{22}\end{array}\right] \notag \\
 & = \left\{\left[\begin{array}{cccc}\lambda_1 & 0 & 0 & 0\\ 0 & \lambda_2 & 0 & 0 \\ 0 & 0 & \lambda_1 & 0 \\ 0 & 0 & 0 & \lambda_2\end{array}\right] - \left[\begin{array}{cccc}\mu_1 & 0 & 0 & 0\\ 0 & \mu_1 & 0 & 0 \\ 0 & 0 & \mu_2 & 0 \\ 0 & 0 & 0 & \mu_2\end{array}\right]\right\} \left[\begin{array}{c}c_{11} \\ c_{21}\\ c_{12} \\ c_{22}\end{array}\right] \notag \\
 & = \left[\begin{array}{cccc}\lambda_1 - \mu_1 & 0 & 0 & 0\\ 0 & \lambda_2 - \mu_1 & 0 & 0 \\ 0 & 0 & \lambda_1 - \mu_2 & 0 \\ 0 & 0 & 0 & \lambda_2 - \mu_2\end{array}\right] \left[\begin{array}{c}c_{11} \\ c_{21}\\ c_{12} \\ c_{22}\end{array}\right]  \notag
 \end{align}
 
 \begin{align}
R.H.S & = \vect (- \rho_1\bar{a}\bar{a}^\top \tilde{C})\notag\\
& = -\rho_1\vect \left(  \left[\begin{array}{c}\bar{a}_1 \\ \bar{a}_2\\\end{array}\right]\left[\begin{array}{cc}\bar{a}_1 & \bar{a}_2\end{array}\right]\left[\begin{array}{cc}c_{11} & c_{12}\\ c_{21} & c_{22}\end{array}\right]\right)\notag \\
& = -\rho_1\vect \left(  \left[\begin{array}{c}\bar{a}_1 \\ \bar{a}_2\\\end{array}\right]\left[\begin{array}{cc}\bar{a}_1c_{11} + \bar{a}_2 c_{21} & \bar{a}_1c_{12} + \bar{a}_2 c_{22
}\end{array}\right]\right)\notag \\
& = -\rho_1\vect \left(  \left[\begin{array}{c}\bar{a}_1 \\ \bar{a}_2\\\end{array}\right]\left[\left[\begin{array}{cc}\bar{a}_1 & \bar{a}_2\end{array}\right]\left[\begin{array}{c}c_{11} \\ c_{21}\end{array}\right] \quad \left[\begin{array}{cc}\bar{a}_1 & \bar{a}_2\end{array}\right]\left[\begin{array}{c}c_{12} \\ c_{22}\end{array}\right]\right]\right)\notag \\
& = -\rho_1\vect \left(  \left[\begin{array}{c}\bar{a}_1 \\ \bar{a}_2\\\end{array}\right]\left[\begin{array}{cc}\bar{a}^\top c_1 & \bar{a}^\top c_2\end{array}\right]\right)\notag \\
& \text{Where, }  \tilde{C}  = \left[\begin{array}{cc}c_1 & c_2\end{array}\right] \text{and }  c_1 = \left[\begin{array}{c} c_{11} \\ c_{21}\end{array}\right] c_2 = \left[\begin{array}{c} c_{12} \\ c_{22}\end{array}\right]\notag\\
& = -\rho_1\vect \left(  \left[\begin{array}{cc}\bar{a}_1\bar{a}^\top c_1 & \bar{a}_1\bar{a}^\top c_2 \\ \bar{a}_2\bar{a}^\top c_1 & \bar{a}_2\bar{a}^\top c_2\end{array}\right]\right)\notag \\
& = -\rho_1\left[\begin{array}{c}\bar{a}_1\bar{a}^\top c_1 \\ \bar{a}_2\bar{a}^\top c_1 \\\bar{a}_1\bar{a}^\top c_2\\ \bar{a}_2\bar{a}^\top c_2\end{array}\right]\notag \\
\end{align}

Equating L.H.S and R.H.S we get,

\begin{align}
\left[\begin{array}{cccc}\lambda_1 - \mu_1 & 0 & 0 & 0\\ 0 & \lambda_2 - \mu_1 & 0 & 0 \\ 0 & 0 & \lambda_1 - \mu_2 & 0 \\ 0 & 0 & 0 & \lambda_2 - \mu_2\end{array}\right] \left[\begin{array}{c}c_{11} \\ c_{21}\\ c_{12} \\ c_{22}\end{array}\right] & = -\rho_1\left[\begin{array}{c}\bar{a}_1\bar{a}^\top c_1 \\ \bar{a}_2\bar{a}^\top c_1 \\\bar{a}_1\bar{a}^\top c_2\\ \bar{a}_2\bar{a}^\top c_2\end{array}\right].\notag
\end{align}
Therefore,
\begin{align}
\left[\begin{array}{c}c_{11} \\ c_{21}\\ c_{12} \\ c_{22}\end{array}\right] & = -\rho_1\left[\begin{array}{cccc}\lambda_1 - \mu_1 & 0 & 0 & 0\\ 0 & \lambda_2 - \mu_1 & 0 & 0 \\ 0 & 0 & \lambda_1 - \mu_2 & 0 \\ 0 & 0 & 0 & \lambda_2 - \mu_2\end{array}\right]^{-1}\left[\begin{array}{c}\bar{a}_1\bar{a}^\top c_1 \\ \bar{a}_2\bar{a}^\top c_1 \\\bar{a}_1\bar{a}^\top c_2\\ \bar{a}_2\bar{a}^\top c_2\end{array}\right]\notag\\
\left[\begin{array}{c}c_{11} \\ c_{21}\\ c_{12} \\ c_{22}\end{array}\right] & = -\rho_1\left[\begin{array}{cccc}\frac{1}{\lambda_1 - \mu_1} & 0 & 0 & 0\\ 0 & \frac{1}{\lambda_2 - \mu_1} & 0 & 0 \\ 0 & 0 & \frac{1}{\lambda_1 - \mu_2} & 0 \\ 0 & 0 & 0 & \frac{1}{\lambda_2 - \mu_2}\end{array}\right]\left[\begin{array}{c}\bar{a}_1\bar{a}^\top c_1 \\ \bar{a}_2\bar{a}^\top c_1 \\\bar{a}_1\bar{a}^\top c_2\\ \bar{a}_2\bar{a}^\top c_2\end{array}\right]\notag\\
\left[\begin{array}{c}c_{11} \\ c_{21}\\ c_{12} \\ c_{22}\end{array}\right] & = -\rho_1\left[\begin{array}{c}\frac{\bar{a}_1\bar{a}^\top c_1}{\lambda_1 - \mu_1} \\ \frac{\bar{a}_2\bar{a}^\top c_1}{\lambda_2 - \mu_1} \\ \frac{\bar{a}_1\bar{a}^\top c_2}{\lambda_1 - \mu_2}\\ \frac{\bar{a}_2\bar{a}^\top c_2}{\lambda_2 - \mu_2}\end{array}\right]\notag\\
\left[\begin{array}{c}c_{11} \\ c_{21}\\ c_{12} \\ c_{22}\end{array}\right] & = \rho_1\left[\begin{array}{c}\frac{\bar{a}_1\bar{a}^\top c_1}{\mu_1 - \lambda_1} \\ \frac{\bar{a}_2\bar{a}^\top c_1}{\mu_1 - \lambda_2} \\ \frac{\bar{a}_1\bar{a}^\top c_2}{\mu_2 - \lambda_1}\\ \frac{\bar{a}_2\bar{a}^\top c_2}{\mu_2 - \lambda_2}\end{array}\right]\notag\\
\text{Where, }\left[\begin{array}{c}c_{11} \\ c_{21}\end{array}\right] & = \rho_1\left[\begin{array}{c}\frac{\bar{a}_1\bar{a}^\top c_1}{\mu_1 - \lambda_1} \\ \frac{\bar{a}_2\bar{a}^\top c_1}{\mu_1 - \lambda_2}\end{array}\right] \quad \text{and } \left[\begin{array}{c}c_{12} \\ c_{22}\end{array}\right]  = \rho_1\left[\begin{array}{c}\frac{\bar{a}_1\bar{a}^\top c_2}{\mu_2 - \lambda_1}\\ \frac{\bar{a}_2\bar{a}^\top c_2}{\mu_2 - \lambda_2}\end{array}\right]\notag\\
\text{i.e. } c_1 & = \rho_1\left[\begin{array}{c}\frac{\bar{a}_1}{\mu_1 - \lambda_1} \\ \frac{\bar{a}_2}{\mu_1 - \lambda_2}\end{array}\right]\bar{a}^\top c_1 \quad \text{and } c_2 = \rho_1\left[\begin{array}{c}\frac{\bar{a}_1}{\mu_2 - \lambda_1}\\ \frac{\bar{a}_2}{\mu_2 - \lambda_2}\end{array}\right]\bar{a}^\top c_2 \notag\\
\text{In general }c_i & = \rho_1\left[\begin{array}{c}\frac{\bar{a}_1}{\mu_i - \lambda_1} \\ \vdots \\ \frac{\bar{a}_2}{\mu_i - \lambda_n}\end{array}\right]\bar{a}^\top c_1
\end{align}

By placing $c_1$ and $c_2$ in $\tilde{C} = \left[\begin{array}{cc} c_1 & c_2\end{array}\right]$ we get $\tilde{C}$ as below.

\begin{align}
\tilde{C} & = \left[\begin{array}{cc} c_1 & c_2\end{array}\right]\notag\\
\tilde{C} & = \left[\begin{array}{cc} \frac{\bar{a}_1\bar{a}^\top c_1}{\mu_1 - \lambda_1} & \frac{\bar{a}_1\bar{a}^\top c_2}{\mu_2 - \lambda_1}\\\frac{\bar{a}_1\bar{a}^\top c_1}{\mu_1 - \lambda_2} & \frac{\bar{a}_1\bar{a}^\top c_2}{\mu_2 - \lambda_2}\end{array}\right]\notag\\
\tilde{C} & = \left[\begin{array}{cc} \bar{a}_1 & 0 \\0 & \bar{a}_2\end{array}\right]\left[\begin{array}{cc} \frac{\bar{a}^\top c_1}{\mu_1 - \lambda_1} & \frac{\bar{a}^\top c_2}{\mu_2 - \lambda_1}\\\frac{\bar{a}^\top c_1}{\mu_1 - \lambda_2} & \frac{\bar{a}^\top c_2}{\mu_2 - \lambda_2}\end{array}\right]\notag\\
\tilde{C} & = \left[\begin{array}{cc} \bar{a}_1 & 0 \\0 & \bar{a}_2\end{array}\right]\left[\begin{array}{cc} \frac{1}{\mu_1 - \lambda_1} & \frac{1}{\mu_2 - \lambda_1}\\\frac{1}{\mu_1 - \lambda_2} & \frac{1}{\mu_2 - \lambda_2}\end{array}\right]\left[\begin{array}{cc} \bar{a}^\top c_1 & 0 \\0 & \bar{a}^\top c_2\end{array}\right]\notag
\end{align}

\section{FAST Algorithm}
	\label{app:fastAlgo}
In this section we present FAST Algorithm \cite{gerasoulis88} that computes functions of the form (\ref{eqn:functionf(x)}) using polynomial interpolation in time $O(n \log^2 n)$.
	\begin{enumerate}
	\item Compute the coefficients of $g(x)$ in its power form, by using FFT polynomial multiplication, in $O(n(\text{log}\ n)^2)$ time.\vspace{0.05in}\\
	
	\textbf{Decription:} Uses FFT for speedy multiplication which reduces complexity of multiplication to $O(n \ \text{log}(n))$ from $O(n^2)$.\\
	\textbf{Input:} Function $g(x)$ and eigenvalues of $\tilde{D}$ i.e. $[\mu_1,\ldots,\mu_n]$\\
	\textbf{Output:} Coefficients of $g(x)$ i.e. $[a_0,a_1,\ldots,a_n]$\\
	\textbf{Complexity:} $O(n \ \text{log}^2n)$
	
	\item Compute the coefficients of $g'(x)$ in $O(n)$ time.\vspace{0.05in}\\
	
	\textbf{Description:} Differentiate the function $g(x)$ and then compute its coefficients.\\
	\textbf{Input:} Function $g(x)$\\
	\textbf{Output:} Coefficients of $g'(x)$ i.e. $[b_0,b_1,\ldots,b_n]$\\
	\textbf{Complexity:} $O(n)$\\
	
	\item Evaluate $g(\lambda_i)$, $g'(\lambda_i)$ and $g(\mu_i)$.\vspace{0.05in}\\
	
	\textbf{Description:} For the functions $g(x)$ and $g'(x)$ evaluate their values at $\lambda_i$ and $\mu_i$\\
		\textbf{Input:} Function $g(x)$ and $g'(x)$, eigenvalues $[\lambda_1,\lambda_2,\ldots,\lambda_n]$ of $D$ and $[\mu_1,\mu_2,\ldots,\mu_n]$ of $\tilde{D}$\\
		\textbf{Output:} $g(\lambda_i)$, $g'(\lambda_i)$ and $g(\mu_i)$\\
		\textbf{Complexity:} $O(n \ \text{log}^2n)$\\
	
	\item Compute $h_j = u_jg'(\lambda_j)$.\vspace{0.05in}\\
	
	\textbf{Description:} For each eigenvalue $\lambda_j$ compute its function value $h(\lambda_j)$ i.e. find the points $(\lambda_j,u_jg'(\lambda_j))$\\
		\textbf{Input:} $[\lambda_1,\lambda_2,\ldots,\lambda_n]$\\
		\textbf{Output:} $h_j$\\
		\textbf{Complexity:} $O(n)$\\
	
	\item Find interpolation polynomial $h(x)$ for the points $(\lambda_j,h_j)$.\vspace{0.05in}\\
	
	\textbf{Description:} Given the function values and input i.e. points $(\lambda_j,h_j)$ find interpolation polynomial for $n$ points.\\
		\textbf{Input:} Points $(\lambda_j,h_j)$\\
		\textbf{Output:} $h(x)$ \\
		\textbf{Complexity:} $O(n \ \text{log}^2n)$\\
	
	\item Compute $v_i = \frac{h_(\mu_i)}{g(\mu_i)}$.\vspace{0.05in}\\
	
	\textbf{Description:} Compute the ratio $v_i = \frac{h_(\mu_i)}{g(\mu_i)}$ for each $\mu_i$ where, $v_i$ is value of $f(\mu_i)$ at each $\mu_i$.\\
		\textbf{Input:} Function $h(x)$ and $g(x)$ and $[\mu_1,\mu_2,\ldots,\mu_n]$\\
		\textbf{Output:}$f(\mu_i) = v_i$\\
		\textbf{Complexity:} $O(n)$\\
	
	\end{enumerate}


\section{Fast Multipole Method}\label{app:FMMAlgortihm}
\subsection{Interpolation and Chebyshev Nodes}
				%
Chebyshev nodes are the roots of the Chebyshev polynomials and they are used as points for interpolation. These nodes lie in the range $[-1,1]$. A polynomial of degree less than or equal to $n-1$ can fit over $n$ Chebyshev nodes. For Chebyshev nodes the approximating polynomial is computed using Lagrange interpolation. 

Expansions are used to quantify interactions among points. Expansions are only computed for points which are well-separated from each other.

\begin{table}[H]
\centering
\begin{tabular}{|c|c|l|}
\hline\rule[-2mm]{0mm}{8mm} No. & Expansion & Description \\
\hline\rule{0pt}{4ex}\multirow{2}{*}{1} & $\Phi_{l,i}$ & Far-field expansion is computed for points\\
&&  withing the cluster/sub-interval $i$ of level $l$.\\ 
\hline\rule{0pt}{4ex}\multirow{2}{*}{2} & $\Psi_{l,i}$ & Local expansion is computed for points which\\
& & are well seperated from sub-interval $i$ of level $l$\\
\hline
\end{tabular}
\caption{Expansions used in FMM }
\label{table:Expansion}
\end{table}

				\subsection{FMM Algorithm}
					\begin{enumerate}
						\item [STEP 1]\textbf{Description:} Decide the size of Chebyshev expansion i.e., $p = -\text{log}_5(\epsilon) = \text{log}_5 (\frac{1}{\epsilon})$ where $\epsilon > 1$ is the precision of computation or machine accuracy parameter.\\
						\textbf{Input:} $\epsilon$\\
						\textbf{Output:} $p$\\
						\textbf{Complexity:} $O(1)$\\
						
						\item [STEP 2]\textbf{Description:} Set $s$ as the number of points in the cell of finest level $(s \approx 2p)$ and level of finest division $nlevs = \text{log}_2(\frac{N}{s})$ where, $N$ is the number of points.\\		
						\textbf{Input:} $N$ and $p$\\
						\textbf{Output:} $s$ and $nlevs$\\
						\textbf{Complexity:} $O(1)$\\	
						
						\item [STEP 3]\textbf{Description:} Consider $p$ Chebyshev nodes defined over an interval $[-1,1]$ of the form below.\\						
						\begin{equation}
						t_i = cos \Big( \frac{2i - 1}{p} \cdot \frac{\pi}{2} \Big)\label{eqn:Define_ti}
						\end{equation}
						Where, $i = 1,\ldots,p$.\\

						\item [STEP 4]\textbf{Description:} Consider Chebyshev polynomials of the form 
						
\begin{equation}
u_j(t) = \prod\limits_{\substack{k=1 \\ k\neq j}}^{p} \frac{t - t_k}{t_j - t_k}.\label{eqn:Define_uj}
\end{equation} 
Where, $j = 1,\ldots,p$.

	\item [STEP 5] \textbf{Description:} Calculate the far-field expansion $\Phi_{l,i}$ at $i^{th}$ subinterval of level $l$.
						
	For $i^{th}$ interval of level $l\ (nlevs)$ far-field expansion due to points in interval $[x_0 - r, x_0 + r]$ about center of the interval $x_0$ is defined by a vector of size $p$ as 
						\begin{equation}
						\Phi_{nlevs,i} = \sum_{k=1}^{N} \alpha_k \cdot \frac{t_i}{3r - t_i(x_k -x_0)}.
						\end{equation}
						\textbf{Input:} $l = nlevs$, $\alpha_k$, $x_k$, $x_0$, $r$, $t_i$ for $i = 1,\ldots,2^{nlevs}$, \\
						\textbf{Output:} $\Phi_{nlevs,i}$\\
						\textbf{Complexity:} $O(Np)$\\
						
							\item [STEP 6]\textbf{[Bottom-up approach]}\\ \textbf{Description:} Compute the far-field expansion of individual subintervals in terms of far-field expansion of their children. These are represented by $p \times p $ matrix defined as below.
												
							\begin{eqnarray}
							M_L(i,j) & = & u_j\Big(\frac{t_i}{2+t_i}\Big)\\
							M_R(i,j) & = & u_j\Big(\frac{t_i}{2-t_i}\Big)
							\end{eqnarray}
				Where, $i = 1,\ldots,p$ and $j = 1,\ldots,p$. $\{u_1,\ldots,u_p\}$ and $\{t_1,\ldots,t_p\}$ are as defined in Eq. (\ref{eqn:Define_uj}) and Eq. (\ref{eqn:Define_ti}) respectively.
				
				Far-field expansion for $i^{th}$ subinterval of level $l$ due to far-field expansion of its children is computed by shifting children's far-field expansion by $M_L$ or $M_R$ and adding those shifts as below. \begin{equation}
				\Phi_{l,i} = M_L \cdot \Phi_{l+1,2i-l} + M_R \cdot \Phi_{l+1,2i}\label{eqn:BottomUp}
				\end{equation}
				\textbf{Input:} $M_L$, $M_R$ and $\Phi$\\
				\textbf{Output:} $\Phi_{l,i}$\\
				\textbf{Complexity:} $O(\frac{2Np^2}{s})$\\
				
				\item [STEP 7] \textbf{Description:} Calculate the local expansion $\Psi_{l,i}$ at $i^{th}$ subinterval of level $l$.
									
				For $i^{th}$ interval of level $l$ local expansion due to points outside the interval $[y_0 - r, y_0 + r]$ about center of the interval $y_0$ is defined by a vector of size $p$ as
				\begin{equation}
				\Psi_i = \sum_{k=1}^{N} \alpha_k \cdot \frac{1}{rt_i - (x_k - x_0)}.
				\end{equation}
				\textbf{Input:} $t_i$, $\alpha_k$, $r$, $x_k$,and  $x_0$\\
				\textbf{Output:} $\Psi_i$\\
				\textbf{Complexity:} $O(Np)$\\
				
				\item [STEP 8]\textbf{[Top-down approach]}\\ \textbf{Description:} Compute the local expansion of individual subintervals in terms of local expansion of their parents. These are represented by $p \times p $ matrix defined as below. 
						
				\begin{eqnarray}
				S_L(i,j) & = & u_j\Big(\frac{t_i - 1}{2}\Big)\\
				S_R(i,j) & = & u_j\Big(\frac{t_i + 1}{2}\Big)
				\end{eqnarray}
				Where, $i = 1,\ldots,p$ and $j = 1,\ldots,p$. $\{u_1,\ldots,u_p\}$ and $\{t_1,\ldots,t_p\}$ are as defined in Eq. (\ref{eqn:Define_uj}) and Eq. (\ref{eqn:Define_ti}) respectively.
				
				Compute local expansions using far-field expansion using $p \times p$ matrix defined as below.
						
				\begin{eqnarray}
				T_1(i,j) & = & u_j\Big(\frac{3}{t_i - 6}\Big)\\
				T_2(i,j) & = & u_j\Big(\frac{3}{t_i - 4}\Big)\\
				T_3(i,j) & = & u_j\Big(\frac{3}{t_i + 4}\Big)\\
				T_4(i,j) & = & u_j\Big(\frac{3}{t_i + 6}\Big)
				\end{eqnarray}
				
				Local expansion for each subinterval at finer level is computed using local expansion of their parents. For this, first the local expansion of parent interval are shifted by $S_L$ or $S_R$ and then the result is added with interactions of subinterval with other well separated subintervals (which were not considered at the parent level).
				
				\begin{equation}
				\Psi_{l+1,2i-1}   =  S_L \cdot \Psi_{l,i} + T_1 \cdot \Phi_{l+1,2i-3} + T_3 \cdot \Phi_{l+1,2i+1} + T_4 \cdot \Phi_{l+1,2i+2}
				\end{equation}	
				
				\begin{equation}
				\Psi_{l+1,2i}   =  S_R \cdot \Psi_{l,i} + T_1 \cdot \Phi_{l+1,2i-3} + T_2 \cdot \Phi_{l+1,2i-2} + T_4 \cdot \Phi_{l+1,2i+2}
				\end{equation}
				\textbf{Input:} $S_L$, $S_R$, $T_1$, $T_2$, $T_3$, $T_4$, $\Phi$ and $\Psi$\\
				\textbf{Output:} $\Psi_{l+1,2i-1}$ and $\Psi_{l+1,2i}$\\
				\textbf{Complexity:} $O(\frac{8Np^2}{s})$\\
				
				\item [STEP 9] \textbf{Description:} Evaluate local expansion $\Psi_{nlevs,i}$ at some of $\{y_j\}$ (which falls in subinterval $i$ of level $nlevs$) to obtain a $p$ size vector. \\
				\textbf{Complexity:} $O(Np)$\\
				
				\item [STEP 10] \textbf{Description:} Add all the remaining interactions which are not covered by expansions. Compute interactions of each $\{ y_k\}$ in subinterval $i$ of level $nlevs$ with all $\{x_j\}$ in subinterval $i-1,i,i+1$. Add all these interactions with the respective local expansions.  \\
				\textbf{Complexity:} $O(3Ns)$\\
			 \end{enumerate}
					
				\begin{box}
1				\underline{\textbf{Total Complexity of FMM}}
				
				\begin{eqnarray*}
				& & O(1+1+Np+(2\frac{2Np^2}{s})+ Np + (8\frac{2Np^2}{s})+Np+3Ns)\\
				& = & O(2+3Np+3Ns+(10Np^2/s))\\
				& = & O(2+3Np+6Np+(10Np^2/2p))\\
				& = & O(2+9Np+(5Np))\\
				& = & O(2+14Np)\\
				& = & O(Np)\\
				& = & O\bigg(N \ \text{log}\bigg(\frac{1}{\epsilon}\bigg)\bigg)
				\text{,\ Where $p = \text{log} \bigg(\frac{1}{\epsilon}\bigg)$}			
				\end{eqnarray*}
				\end{box}

\end{document}